\documentclass[twoside]{article}

\usepackage[accepted]{aistats2020}
\usepackage[round]{natbib}

\usepackage[subtle]{savetrees}

\usepackage{braket} 
\usepackage[pdftex]{graphicx} 
\usepackage{amsfonts}
\usepackage[shortlabels]{enumitem}
\usepackage{amsmath}
\usepackage{amssymb}
\usepackage{amsthm} 
\usepackage[makeroom]{cancel}
\usepackage{dsfont}
\usepackage{multirow}
\usepackage{algorithm}
\usepackage{algorithmic}
\usepackage{subcaption}
\newtheorem{definition}{Definition}
\newtheorem{theorem}{Theorem}
\usepackage{amsthm}
\usepackage{tikz} 
\usepackage{filecontents}  
\usepackage{graphicx}
\usepackage{transparent}

\newtheorem{prop}{Proposition}

\begin{document}

%

%

\twocolumn[

\aistatstitle{Expressiveness and Learning of Hidden Quantum Markov Models}

\aistatsauthor{ Sandesh Adhikary*\\
University of Washington
\And Siddarth  Srinivasan*\\Microsoft Research, Montr\'{e}al\\
Georgia Institute of Technology
\AND  Geoff Gordon\\Microsoft Research, Montr\'{e}al 
\And Byron Boots\\
University of Washington\\}
\aistatsaddress{}



  

  ]

\begin{abstract}
    Extending classical probabilistic reasoning using the quantum mechanical view of probability has been of recent interest, particularly in the development of hidden quantum Markov models (HQMMs) to model stochastic processes. However, there has been little progress in characterizing the expressiveness of such models and learning them from data. We tackle these problems by showing that HQMMs are a special subclass of the general class of observable operator models (OOMs) that do not suffer from the negative probability problem by design.
    We also provide a feasible retraction-based learning algorithm for HQMMs using constrained gradient descent on the Stiefel manifold of model parameters. We demonstrate that this approach is faster and scales to larger models than previous learning algorithms.
\end{abstract}
\section{Introduction and Related Work}
Classical probabilistic graphical models provide a principled framework for Bayesian reasoning, and there has been much interest in extending this framework by incorporating the mathematical formalism of quantum mechanics \citep{leifer2008quantum, yeang2010probabilistic, leifer2013towards, warmuth2014bayesian}. Hidden quantum Markov models (HQMMs) 
\citep{monras2010hidden, clark2015hidden, srinivasan2018}, have been some of the more well-investigated models; recent work by \citet{srinivasan2018} showed that every finite-dimensional hidden Markov model (HMM) can also be modeled by a finite-dimensional HQMM, and empirically demonstrated some theoretical advantages of HQMMs over HMMs.  A major motivation for investigating such `quantum models' has been the promise of a more general and expressive class of probabilistic models. Yet, a clear characterization of the expressiveness of these models and a practical learning algorithm has remained lacking. These are precisely the problems we tackle in this paper.\let\thefootnote\relax\footnotetext{* denotes equal contribution}

Our theoretical exploration of HQMMs is primarily centered around their relationship to the observable operator models (OOMs) developed by \citet{jaeger2000observable}. OOM-equivalents have been independently developed and are also referred to in the literature as uncontrolled predictive state representations (PSRs) \citep{SinghPSR}, linearly dependent processes \citep{ito1992identifiability}, and stochastic weighted automata \citep{balle2014methods, Thon2015}. OOMs can be seen as a generalization of the well-known hidden Markov models \citep{rabiner1986introduction}, but despite their generality they lack a constructive definition. A valid OOM must never produce a negative probability for a sequence of observations, yet it is \emph{undecidable} \citep{Wiewiora2007ModelingPD} whether or not candidate set of OOM parameters will yield negative probabilities. This is known as the \emph{negative probability problem} (NPP) of OOMs, and must be handled with heuristics in practice \citep{cohen2013experiments}.  
An alternative approach is to construct models that avoid the NPP by design, such as norm-observable operator models (NOOMs) \citep{Zhao2010NormObservableOM} or quadratic weighted automata \citep{bailly2011quadratic}. While NOOMs can simulate processes that no finite-dimensional HMM could model (such as the `probability clock' \citep{Zhao2010NormObservableOM}), it is unclear whether they have the broad expressiveness of OOMs; it isn't even known if they contain HMMs as a subclass. In this context, we make three main theoretical contributions in this paper: (i) we show how  HQMMs can be seen as a generalization of NOOMs, (ii) we formulate the Liouville representation of HQMMs which uniquely characterizes the model and allows for direct comparison between HQMMs, and (iii) we show that every finite-dimensional HQMM is equivalent to a finite-dimensional OOM, with the special property that we can characterize the valid initial states as the spectraplex of Hermitian PSD matrices with trace 1.

We also present results on learning these models from data. We use the Kraus operator parameterization of HQMMs using matrices $\{\mathbf{K}_i\}$ that satisfy the constraint $\sum_i \mathbf{K}_i^\dagger \mathbf{K}_i = \mathds{I}$. Stacking the operators $\mathbf{K}_i$ vertically to form a matrix $\pmb{\kappa}$, the constraint can be re-written as $\pmb{\kappa}^\dagger \pmb{\kappa} = \mathds{I}$. The existing approach to learning HQMMs~\citep{srinivasan2018} yields feasible parameters by starting with an initial guess $\mathbf{\kappa}$ and iteratively finding unitary transformations that increase the likelihood of the data. However, this method is inefficient, often gets trapped in poor optima, and can only handle a small number of hidden states. The absence of a practical learning algorithm has been a bottleneck in the development of these models \citep{schuld2015introduction}. Our primary experimental contribution in this paper is the application and analysis of a viable approach to the learning problem: since $\pmb{\kappa}$ lies on the Stiefel manifold \citep{stiefel-original,Edelman1998}, we can directly learn feasible parameters by constraining gradient updates to lie on the manifold using a well-known retraction-based algorithm \citep{Wen2013}. We show that this approach is faster, finds better optima, and can handle more hidden states than the previous method. 

\section{The Expressiveness of HQMMs}
In general, the models we discuss are used to model sequential data and assume an evolving latent state that emits discrete observations at each time-step. We describe HMMs, OOMs, and NOOMs, and show how HQMMs can be derived as a generalization of NOOMs.  
\subsection{Hidden Markov Models} 
\vspace{-2mm}
\begin{definition}[HMMs]
\label{def:hmm}
An $n$-dimensional Hidden Markov Model with a set of discrete observations $\mathcal{O}$ is a tuple ($\mathds{R}^n$, $\mathbf{A}$, $\mathbf{C}$, $\vec{x}_0$) where initial state $\vec{x}_0$, transition matrix $\mathbf{A}$, and emission matrix $\mathbf{C}$ satisfy the following conditions:
\begin{enumerate}[(i)]
    \item Non-negative parameters: $\vec{x}_0 \in \mathds{R}^n_{\geq 0}$, $\mathbf{A} \in \mathds{R}^{n \times n}_{\geq 0}$, $\mathbf{C} \in \mathds{R}^{|\mathcal{O}| \times n}_{\geq 0}$, 
    \item Normalized initial state: $\vec{1}^{T}\vec{x}_0 = 1$,
    \item Column-stochastic operators: $\vec{1}^T\mathbf{A} = \vec{1}^T\mathbf{C} = \vec{1}^T$.
\end{enumerate}
\end{definition}

HMM belief states are always interpretable as probability distributions over hidden system states. 

At each time-step, we update the belief state and condition on observation using the column-stochastic matrices ${\bf A}$ and ${\bf C}$ respectively:
\begin{equation}\label{eq:hmm}
    \vec{x}_{t}' = {\bf A}\vec{x}_{t-1}
    ~~~~~~~~~~~~~
    \vec{x}_{t} = \frac{\text{diag}({\bf C}_{(y,:)})\vec{x}_{t}'}{\vec{1}^T \text{diag}({\bf C}_{(y,:)})\vec{x}_{t}'},
\end{equation}
where diag$\left(\mathbf{C}_{y,:}\right)$ places the row $y$ of matrix $\mathbf{C}$ in a diagonal matrix. We can also compute the probability of a sequence of observations $\bar{y} = y_1, \ldots, y_t$ from a given belief state $\vec{x}$ as follows:
\begin{small}
\begin{equation}
    P(\bar{y}) = \vec{1}^T \text{diag}({\bf C}_{(y_t,:)})\mathbf{A}~\cdots~\text{diag}({\bf C}_{(y_1,:)})\mathbf{A}\vec{x}
\end{equation}
\end{small}
\vspace{-8mm}
\subsection{Observable Operator Models} 
We describe OOMs as a generalization of HMMs. Observe that the operations above can be equivalently represented by defining observable operators $\mathbf{T}_y = \text{diag}({\bf C}_{(y,:)})\mathbf{A}$ for each observation $y$:
\begin{small}
\begin{equation}
\label{obs:oom}
    \vec{x}_{t} = \frac{{\bf T}_y\vec{x}_{t-1}}{\vec{1}^T {\bf T}_y\vec{x}_{t-1}}
       ~~~~~~~~~~~~~
    P(\bar{y}) = \vec{1}^T {\bf T}_{y_t}~\cdots~{\bf T}_{y_1} \vec{x}
\end{equation}
\end{small}
We can arrive at OOMs by relaxing constraint (i) in Definition \ref{def:hmm} (so entries in $\vec{x}$, $\mathbf{A}$, $\mathbf{C}$ can be negative) and requiring only that the model always assign non-negative probabilities to observations. This allows us to define a standard OOM as follows:

\begin{definition}[Standard OOMs \citep{jaeger2000observable}]
\label{def:s_ooms} An $n$-dimensional standard Observable Operator Model with a set of discrete observations $\mathcal{O}$ is a tuple $(\mathbb{R}^n, \{\mathbf{T}_y\}_{y \in \mathcal{O}}, \vec{x}_0)$ where initial state $\vec{x}_0 \in \mathbb{R}^n$ and observable operators $\{\mathbf{T}_y\}_{y \in \mathcal{O}} \in \mathbb{R}^{n \times n}$ satisfy the following constraints:

\begin{enumerate}[(i)]
\item Normalized initial state: $\vec{1}^T \vec{x}_0 = 1$, 
\item Normalized marginal over observations: $\vec{1}^T\sum_{y \in \mathcal{O}} \mathbf{T}_y = \vec{1}^T$, 
\item Non-negative probabilities: $\vec{1}^T \mathbf{T}_{y_{t}}\dots\mathbf{T}_{y_{1}} \vec{x}_0 \geq 0$ for all sequences $y_{1} \dots y_{t}$.
\end{enumerate}
\end{definition}
\vspace{-2mm}
Note that the above definition is non-constructive since it does not tell us what constraints we could place on model parameters or initial states to satisfy condition (iii) -- this is the cost of relaxing the non-negativity constraint. 

In fact, it is \emph{undecidable} whether a given candidate OOM ($\mathbb{R}^n$, $(\mathbf{T}_y)_{y \in \mathcal{O}}$, $\vec{x}_0$) satisfying conditions (i)-(ii) will violate condition (iii) \citep{Wiewiora2007ModelingPD}. This is the root of the infamous negative probability problem (NPP) in OOMs, since we cannot identify whether a learned model will assign negative probabilities to observations. %

\citet{jaeger2000observable} further showed that HMM $\subset$ OOM using the `probability clock' OOM which requires an infinite-dimensional HMM to model. The non-negativity constraint (i) from Definition \ref{def:hmm} forces the largest eigenvalue of an observable operator $\mathbf{T}_y$ of an HMM to be real (by the Perron-Frobenius theorem). However, negative entries in OOMs allow the largest eigenvalue to be complex, which allows the latent states (and hence conditional probabilities) to display oscillatory behaviour. \citet{jaeger2000observable} uses this property in their probability clock example.

A useful conceptual characterization of a candidate OOM with parameters $\{\mathbf{T}_y\}_{y\in\mathcal{O}}$ is the convex cone of valid initial states it admits, i.e., the initial states for which the model will never assign a negative probability for observations. If there is no such cone, the model is invalid. Indeed, \cite{jaeger2000observable} present the following alternative to condition (iii):
\begin{prop}[\citet{jaeger2000observable}] A tuple $(\mathbb{R}^n, (\mathbf{T}_y)_{y \in \mathcal{O}}, \vec{x}_0)$ satisfying conditions (i)-(ii) of Definition \ref{def:s_ooms} is an OOM if and only if there exists a pointed convex cone $K$ such that:
\begin{enumerate}[(i)]
    \item Initial state is in the cone: $\vec{x}_0 \in K$,
    \item Cone is closed under the operators: $\mathbf{T}_y\vec{x} \in K$ for all $\vec{x} \in K$ and $y \in \mathcal{O}$,
    \item 
    The sum of entries for any point in the cone is non-negative:
    $\vec{1}^T\vec{x} \geq 0$ for all $\vec{x} \in K$.
    \end{enumerate}
\label{prop:oom_cones}
\end{prop}
Conditions (i) and (ii) guarantee that any initial state inside such a cone will stay inside the cone under action of $\mathbf{T}_y$, and condition (iii) guarantees that any state inside the cone will evaluate to a non-negative probability. This characterization can also tell us which OOMs have equivalent HMMs: a finite-dimensional OOM has an equivalent finite-dimensional HMM if and only if $K$ is a $k$-polyhedral cone for some $k$, i.e., it is generated by some finite set of vectors \citep{jaeger2000observable}. Proposition~\ref{prop:oom_cones} also gives us a recipe to find OOMs that do not suffer from the NPP: select a desired convex cone of valid initial states and construct operators such that the cone is closed under their action.

\paragraph{General OOMs} The standard OOMs given in Definition \ref{def:s_ooms} are the original formulation by \citet{jaeger2000observable}, which is stricter than necessary. Various equivalent formulations have been proposed, including as Sequential Systems (SS) by \citet{Thon2015}, uncontrolled predictive state representations (PSRs), or stochastic weighted automata \citep{balle2014methods}. In this paper, we refer to the these as `general OOMs'. The main difference is that the model parameters are no longer constrained to be real, and we don't force the state entries to sum to one; instead the state can be any vector as long as we can use a linear functional $\sigma$ (which for standard OOMs was fixed to be $\vec{1}^T$) to recover the probabilities. While the model parameters can be defined over arbitrary fields, we define general OOMs over the complex field as this allows us to eventually recover HQMMs.

\begin{definition}[General OOMs \citep{Thon2015}] 
\label{def:ldp}
An $n$-dimensional general Observable Operator Model with a set of discrete observations $\mathcal{O}$ is a tuple $(\mathbb{C}^n, (\pmb{\tau}_y)_{y \in \mathcal{O}}, \vec{x}_0, \sigma)$ where initial state $\vec{x}_0 \in \mathbb{C}^n$, observable operators $\{\mathbf{\tau}_y\}_{y \in \mathcal{O}} \in \mathds{C}^{n \times n}$, and a linear evaluation functional $\vec{\sigma} \in \mathds{C}^{n}$ satisfy the following constraints:
\begin{enumerate}[(i)] 
\item Normalized Initial State: $\vec{\sigma}^\dagger\vec{x}_0 = 1 $, 
\item Normalized marginal over observations:
$\vec{\sigma}^\dagger\pmb{\tau}_{y_t}\ldots\pmb{\tau}_{y_1}x_0$ = $\sum_{y \in \mathcal{O}} \vec{\sigma}^\dagger\pmb{\tau}_y\pmb{\tau}_{y_t}\ldots\pmb{\tau}_{y_1}\vec{x}_0$ for all sequences ${y_1}\ldots {y_t}$,
\item Non-negative probabilities: $\vec{\sigma}^\dagger \pmb{\tau}_{y_{t}}\dots\pmb{\tau}_{y_{1}} \vec{x}_0 \in [0, 1]$ for all sequences $y_{1} \dots y_{t}$.
\end{enumerate}
\end{definition}
For such a model, the state update after observing $y \in \mathcal{O}$ and computing the probability of that observation are carried out as follows:\footnote{$\dagger$ is the complex conjugate transpose}
\begin{equation}
\label{soom}
    \vec{x}_{t} = \frac{{\pmb{\tau}}_y\vec{x}_{t-1}}{{\vec{\sigma}}^\dagger {\pmb{\tau}}_y\vec{x}_{t-1}}
       ~~~~~~~~~~~~~
    P(\bar{y}) = \vec{\sigma}^\dagger \pmb{\tau}_{y_t} ~\cdots~ \pmb{\tau}_{y_1} \vec{x}
\end{equation}
As shown in Proposition 13 of \citet{Thon2015}, every $n$-dimensional general OOM has an equivalent standard OOM that is a similarity transform away, i.e., we can find a similarity transform $\mathbf{S}$ such that $(\mathbb{C}^n, (\mathbf{S}\,\pmb{\tau}_y \mathbf{S}^{-1})_{y \in \mathcal{O}},\mathbf{S}\,\vec{\omega}_0, \vec{\sigma}\, \mathbf{S}^{-1}) = (\mathbb{C}^n, (\mathbf{T}_y)_{y \in \mathcal{O}}, \vec{v}_0, \vec{1}^T)$.  
We will use this equivalence to show that NOOMs and HQMMs are special cases of OOMs.
Finally, we note that finite dimensional OOMs are the most expressive class of linear models capable of modeling any stochastic process whose `system-dynamics' matrix \citep{SinghPSR} has finite rank \citep{Zhao2010b}. Hence these models are extremely powerful, although the NPP makes it challenging to use these models in practice. 
\vspace{-2mm}
\subsection{Norm-observable Operator Models}\vspace{-2mm}
NOOMs represent a class of models designed to avoid the NPP by construction. The central idea is to wrap the output of the model with the non-linear function $\|\cdot\|^2$ so that it always returns non-negative values.  

\begin{definition}[NOOMs \citep{Zhao2010b}]
\label{def:nooms}
An $n$-dimensional Norm Observable Operator Model with a set of discrete observations $\mathcal{O}$ is a tuple $(\mathbb{R}^n, (\pmb{\phi}_y)_{y \in \mathcal{O}}, \vec{v}_0)$ where initial state $\vec{v}_0 \in \mathbb{R}^n$ and observable operators $\{\pmb{\phi}_y\}_{y \in \mathcal{O}} \in \mathbb{R}^{n \times n}$ satisfy the following constraints:
\begin{enumerate}[(i)]
\item Normalized initial state: $\|\vec{v}_0\|_2^2 = 1$, 
\item Normalized marginal over observations: $\sum_{y \in \mathcal{O}}~\pmb{\phi}_y^\dagger~\pmb{\phi}_y = \mathbb{I}$.
\end{enumerate}
\end{definition}
The updated state after observing $y \in \mathcal{O}$ and the probability of that observation can be computed as
 \begin{small}
\begin{equation}
    \vec{v}_{t} = \frac{{ \pmb{\phi}}_y\vec{v}_{t-1}}{\| \pmb{\phi}_{y_t}~\dots~ \pmb{\phi}_{y_1} \vec{v}\|}
       ~~~~~~~~~~~~~
    P(\bar{y}) = \| \pmb{\phi}_{y_t}~\dots~ \pmb{\phi}_{y_1} \vec{v}\|^2
\label{eq:NOOM}
\end{equation}
\end{small}
Although any stochastic process can be represented as a NOOM in some inner product space, this space may be infinite dimensional \citep{Zhao2010b}. For practical purposes, we care about the expressiveness of finite-dimensional NOOMs. \citet{Zhao2010b} showed that NOOM $\subseteq$ OOM, and once again used the ability of a real-valued NOOM operator to have complex eigenvalues in a NOOM probability clock to show that there are finite-dimensional NOOMs that cannot be modeled exactly by finite-dimensional HMMs. 

\citet{Zhao2010b} show that despite its non-linear form, NOOMs are equivalent to $n^2$-dimensional OOMs, and indeed we will build upon this approach to re-derive HQMMs. \citet{Zhao2010NormObservableOM} use Kronecker product relationships for the 2-norm (where $\vec{\mathds{I}}$ is a vectorized identity matrix that implements a matrix trace operation) to show that sequence probabilities in a NOOM from Equation \ref{eq:NOOM} can also be evaluated as:
\begin{align}
    P(\bar{y}) = \vec{\mathds{I}}_{n^2}^T ~\left(\pmb{\phi}_{y_n}\otimes~\pmb{\phi}_{y_n}\right)~\ldots~\left(\pmb{\phi}_{y_1} \otimes~\pmb{\phi}_{y_1}\right) \left(\vec{v}_0~\otimes~\vec{v}_0\right),
    \label{eq:noom_kron}
\end{align}
Now, if we define $\vec{\sigma} = \vec{\mathds{I}}_{n^2}$,
$\pmb{\tau}_y = \pmb{\phi}_{y} \otimes \pmb{\phi}_{y}$, and the initial state $\vec{\omega}_0 \in \mathds{R}^{n^2}$ as $\vec{\omega}_0 = \vec{v}_0~\otimes~\vec{v}_0$, we get a general OOM ($\mathbb{C}^n$, $(\pmb{\tau}_y)_{y \in \mathcal{O}}$, $\vec{\omega}_0$, $\vec{\sigma}$). As shown by \citet{Zhao2010b}, this is a similarity transform of a standard OOM, with $\mathbf{S} = \mathds{I}_{n^2} + \frac1{n^2}\vec{\mathbf{1}}_{n^2}(\vec{\sigma}^T - \vec{\mathbf{1}}^T_{n^2})$. Thus, NOOMs are not any more expressive than OOMs, i.e., NOOM $\subseteq$ OOM. 
\subsection{Hidden Quantum Markov Models }
\label{sec:hqmms}
Previous work by \citet{srinivasan2018} derived HQMMs by generalizing HMMs using system-environment interactions (illustrated using a quantum circuit), and showed that every $n$-dimensional HMM can be modeled by an HQMM with no more than an $n^2$-dimensional hidden states. Here, we take a different approach; we will show how HQMMs can be defined through a series of natural generalizations of NOOMs in such a way that they also end up containing finite-dimensional HMMs. We do so by allowing parameters to be complex and expanding the concepts of NOOM states and operators using the representation in Equation \ref{eq:noom_kron}.

\paragraph{Generalizing NOOM States}

We know from Equation~\ref{eq:noom_kron} that the initial state $\vec{\omega}_0$ can viewed as a vectorized rank-1 Hermitian matrix $\pmb{\rho}$, i.e., $\vec{\omega} = \text{vec}\left(\vec{v}_0\vec{v}_0^\dagger\right) = \text{vec}(\pmb{\rho})$. A natural generalization would be to let the initial state be a vectorized matrix of arbitrary rank, i.e., $\pmb{\rho}_0 = \sum_i p_i \vec{v}_i \vec{v}_i^\dagger$ instead. The normalization condition on the initial state can then be restated as $1 = \vec{\sigma}^\dagger \vec{\rho}_0 = \vec{\mathds{I}}_{n^2}^T \vec{\rho}_0 = \text{tr}\left(\pmb{\rho}_0\right) = \sum_i p_i$. 

As a linear combination of outer products of vectors with themselves, $\pmb{\rho}$ must be Hermitian. We additionally assume that the constituent eigenvectors live in a Hilbert space $\mathcal{H}$, so that $\pmb{\rho}$ lives in a Liouville space, i.e., the outer product of two Hilbert spaces. Further, in the NOOM, $\vec{v}_0\vec{v}_0^\dagger$ had a single eigenvalue of 1. If we impose no further constraints, we could allow $p_i$ to be complex-valued or negative as long as the normalization condition above was satisfied. However, this could once again lead to negative probabilities when applying the evaluation $\vec{\sigma}$, and hence a non-constructive model. Thus, we impose a positive semi-definiteness (PSD) constraint on the initial state to guarantee that $p_i \in \mathds{R}_{\geq 0}$ so that $\text{tr}(\pmb{\rho}_0)$ is real and non-negative. Essentially, we are now considering a model whose initial states $\vec{\rho}$ are vectorized arbitrary-rank Hermitian PSD matrices, which constitute a pointed convex cone. Such matrices are called \emph{density matrices} in quantum mechanics \citep{nielsen_chuang_2010}, and the imposition of the PSD constraint on the states is what allows these models to avoid the NPP.
\vspace{-2mm}
\paragraph{Generalizing NOOM operators}
Having defined a convex cone of valid states, we now derive operators that ensure that the state always evolves inside the cone. We refer to such operators acting on our states in Liouville space as \emph{Liouville superoperators} $\{\mathbf{L}_y\}_{y \in \mathcal{O}}$. Condition (ii) in Definition \ref{def:nooms} ensured that probabilities of observations computed by the NOOM were normalized, and the equivalent condition in the OOM representation in Equation \ref{eq:noom_kron} is that $\vec{\sigma}^\dagger \left(\sum_{y\in\mathcal{O}} \pmb{\tau}_y\right) = \vec{\sigma}^\dagger$. We impose a similar constraint (trace preservation or TP) on the superoperators to ensure we get a normalized distribution over observations. In addition to this, we further need to ensure that the probabilities assigned to observations are real and non-negative, i.e., the operators must always preserve the Hermitian PSD condition of the state. Finding a constructive way to impose these restrictions on Liouville superoperators is challenging, and it is easier to do so on the `reshuffled' version of it called its \emph{Choi matrix} \citep{Wood2011}. The reshuffle operation (Figure \ref{fig:cp_map_forms}) involves reshaping the $n^2-$dimensional columns of the Liouville superoperator into $n \times n$ matrices. Going across the columns of $\mathbf{L}$ from left to right, we fill up the blocks of the Choi matrix column-first with these reshaped matrices (see \citet{Zyczkowski2004} for further details). In the context of Hermitian preserving (HP) maps, there is no elegant way to also impose a simple PSD-preserving `positivity' constraint \citep{Choi1975CompletelyPL, Pillis1967}. Therefore, we must impose a slightly more restrictive complete positivity (CP) constraint which guarantees that the map $\mathbf{L}_y \otimes \mathds{I}$ is PSD-preserving for identity matrices of any dimension. In fact, \citet{Choi1975CompletelyPL} suggest that a CP map is the natural constructive generalization of `positivity' for a linear HP map. We define L-HQMMs as a generalization of NOOMs with these constraints:
\begin{definition}[L-HQMMs]
\label{def:lhqmm}
An $n^2$-dimensional Liouville-Hidden Quantum Markov Model with a set of discrete observations $\mathcal{O}$ is a tuple $(\mathbb{C}^{n^2}, (\mathbf{L}_y)_{y \in \mathcal{O}}, \vec{\rho}_0, \vec{\mathds{I}})$ where the initial state $\vec{\rho}_0 \in \mathbb{C}^{n^2}$ and Liouville superoperators $\{\mathbf{L}_y\}_{y \in \mathcal{O}} \in \mathds{C}^{n^2 \times n^2}$ with corresponding Choi matrices $\{\mathbf{C}_y\}_{y\in\mathcal{O}}$ satisfy the following constraints:
\begin{enumerate}[(i)]
\item $\vec{\rho}_0$ is a vectorized Hermitian PSD matrix of arbitrary rank, 
\item Normalized initial state: $\vec{\mathds{I}}^T\vec{\rho}_0 = 1 $, 
\item  CP: $\mathbf{C}_y \geq 0$ (Choi matrix is PSD).
\item TP:  $\vec{\mathds{I}}^T\left(\sum_{y\in \mathcal{O}} \mathbf{L}_y\right) = \vec{\mathds{I}}^T$, 
\item HP: $\mathbf{C}_y = \mathbf{C}_y^\dagger$,
\end{enumerate}
\end{definition}
For such a model, the state update after observing $y \in \mathcal{O}$ and computing the probability of that observation are:
\begin{small}
\begin{equation}
\label{loom}
    \vec{\rho}_{t} = \frac{{ \mathbf{L}}_y\vec{\rho}_{t-1}}{{\vec{\mathds{I}}}^T {\mathbf{L}}_y\vec{\rho}_{t-1}}
       ~~~~~~~~~~~~~
    P(\bar{y}) = \vec{\mathds{I}}^T \mathbf{L}_{y_t}~\dots~ \mathbf{L}_{y_1} \vec{\rho}
\end{equation}
\end{small}
The exact relationship between HQMMs and OOMs was previously unknown, but this formulation of HQMMs allows us to state an important result:
\begin{theorem}
HQMM $\subseteq$ OOM, and the set of valid initial states for HQMMs is a spectraplex.
\label{theorem:LHQMM_OOM}
\end{theorem}
\begin{proof} Setting $\vec{\sigma} = \vec{\mathds{I}}$, L-HQMMs satisfy condition (i) of General OOMs laid out in Definition \ref{def:ldp} by construction. Condition (ii) of Definition \ref{def:ldp} is satisfied by the TP constraint on L-HQMMs. Next, the HP and CP constraints on L-HQMMs guarantee that $\mathbf{L}_{\bar{y}}\vec{\rho}$ always yields a vectorized Hermitian PSD matrix. The trace of this matrix is always real and non-negative, i.e., $\vec{\mathds{I}}^T\mathbf{L}_{\bar{y}}\vec{\rho} \geq 0$. We also have $\vec{\mathds{I}}^T\mathbf{L}_y\vec{\rho}_0 \leq \vec{\mathds{I}}^T\left(\sum_{y\in \mathcal{O}} \mathbf{L}_y\right) \vec{\rho}_0 = \vec{\mathds{I}}^T\vec{\rho}_0 = 1$, satisfying condition (iii) of Definition \ref{def:ldp}.

The valid initial states of L-HQMMs are Hermitian PSD matrices with unit trace. Hermitian PSD matrices form a convex cone, and the intersection of this cone with the linear affine subspace of trace 1 matrices is a spectrahedron known as a spectraplex.
\end{proof}
Using the same similarity transform that we used for NOOMs $\mathbf{S} = \mathds{I}_{n^2} + \frac1{n^2}\vec{\mathbf{1}}_{n^2}(\vec{\sigma}^T - \vec{\mathbf{1}}^T_{n^2})$, we can transform any $n^2$-dimensional L-HQMM into an equivalent standard OOM. 

It is still an open question whether HQMMs are a proper subset of OOMs.

\vspace{-2mm}
\paragraph{An alternate formulation of HQMMs} 
Prior work on HQMMs have represented these models in the so-called operator-sum representation \citep{srinivasan2018,monras2010hidden}. While the notion of operating on vectorized matrices is fairly common in quantum information (and was implicitly used for HQMMs in \citet{NIPS2018_8235}), L-HQMMs are a novel formulation of HQMMs. We now derive the operator-sum representation of HQMMs from L-HQMMs, showing that the two are equivalent.

From Definition~\ref{def:lhqmm}, we know that any model equivalent to L-HQMMs must have CP, TP, and HP operators. From Choi's theorem \citep{Choi1975CompletelyPL}, we know that any map which can be expressed in the operator-sum representation $\mathcal{K}(\pmb{\rho}) = \sum_w \mathbf{K}_w~\pmb{\rho}~\mathbf{K}_w^\dagger$ is guaranteed to be CP, and will preserve the PSD nature of any input matrix. In the context of CP maps, the operator matrices $\mathbf{K}_w$ are commonly called Kraus operators \citep{Kraus1971}. The quadratic application of operator preserves the Hermiticity of $\pmb{\rho}$. Thus, the operator-sum representation is particularly appealing because it guarantees the CP and HP constraints by construction. Note that this representation of CP maps is merely a vectorization of the Liouville form
\begin{align*}
    \text{vec}\left( \sum_w~\mathbf{K}_w ~\pmb{\rho}~\mathbf{K}_w^\dagger \right)~=~\sum_w~(\mathbf{K}_w^*~\otimes~ \mathbf{K}_w) \vec{\rho}~=~ \mathbf{L} \vec{\rho}
\end{align*}
Thus, the action of a Liouville superoperator $\mathbf{L}_y$ corresponding to the observable $y$ on $\vec{\rho}$ can can be equivalently represented by a set of Kraus operators $\{\mathbf{K}_{y,w_y}\}$ acting on the density matrix $\pmb{\rho}$, where the cardinality of this set $|w_y|$ is determined by the Schmidt-rank (or Kraus-rank, as we soon explain) of $\mathbf{L}_y$. The Schmidt-rank is analogous to the rank revealed by an SVD, but for a decomposition into a Kronecker product of two vector spaces.

Finally, the operator-sum representation also provides a convenient constraint way of ensuring the TP constraint: the full set of Kraus operators across all observables must satisfy $\sum_{y,w_y} \mathbf{K}_{y,w_y}^\dagger \mathbf{K}_{y,w_y} = \mathds{I}$ \citep{nielsen_chuang_2010}. Note that this condition essentially generalizes condition (ii) for NOOMs in Definition~\ref{def:nooms} to allow multiple operators per observable. We can now define HQMMs using the Kraus operator-sum representation, as given in \citet{srinivasan2018}.
\begin{definition}[K-HQMMs]
\label{def:khqmm}
An $n$-dimensional Kraus-Hidden Quantum Markov Model with a set of discrete observations $\mathcal{O}$ is a tuple $(\mathbb{C}^{n\times n}, \{\mathbf{K}_{y,w_y}\}_{y \in \mathcal{O}}, \pmb{\rho}_0, \text{tr}(\cdot))$ where initial state $\pmb{\rho}_0 \in \mathbb{C}^{n \times n}$ and Kraus operators $\{\mathbf{K}_{y, w_y}\}_{y \in \mathcal{O}, w_y \in \mathds{N}} \in \mathds{C}^{n \times n}$ satisfy the following constraints:
\begin{enumerate}[(i)]
    \item  $\pmb{\rho}_0$ is a Hermitian PSD matrix of arbitrary rank, 
    \item Normalized Initial State: $\text{tr}(\pmb{\rho}_0) = 1 $, 
    \item Normalized marginal over observations (TP): $\sum_{y, w} \mathbf{K}_{y,w}^\dagger \mathbf{K}_{y,w} = \mathds{I}$.
\end{enumerate}
\end{definition}

The state update after observing $y$ is computed as
\begin{equation}
    \pmb{\rho}_t = \frac{\sum_{w_y}\mathbf{K}_{y,w_y} ~\pmb{\rho}_{t-1}~ \mathbf{K}_{y,w_y}^\dagger}{\text{tr}\left(\sum_{w_y} \mathbf{K}_{y,w_y} ~\pmb{\rho}_{t-1}~ \mathbf{K}_{y,w_y}^\dagger\right)} 
    \label{eq:hqmm_joined},
\end{equation}
and probability of a given sequence is given by:
\begin{small}
\begin{equation}
    P(\bar{y}) = \text{tr}\left(\sum_{w_{y_t}} \mathbf{K}_{y_t,w_{y_t}} \ldots \left(\sum_{w_{y_t}} \mathbf{K}_{y_1,w_{y_1}} \pmb{\rho}_{0} \mathbf{K}_{y_1,w_{y_1}}^\dagger\right) \ldots \mathbf{K}_{y_t,w_{y_t}}^\dagger\right)
\end{equation}
\end{small}
The K-HQMM representation was used by \citet{srinivasan2018} to show that any $n$ dimensional HMM can be written as an equivalent $n^2$ dimensional K-HQMM, while there were HQMMs like the NOOM probability clock (trivially a HQMM) that required infinite-dimensional HMMs; hence HMM $\subset$ HQMM.
\vspace{-2mm}
\paragraph{Uniqueness of L-HQMMs}
Note that the Kraus operator sum formulation of K-HQMMs does not uniquely define a CP map; it can be equivalently defined using different sets of Kraus operators (with possibly different cardinalities). Thus, it is not evident how one might compare two K-HQMMs. On the other hand, the Liouville superoperator is the unique representation of a CP map, and can be canonically factorized as follows \citep{Wood2011,Miszczak2011SingularVD}:
\vspace{-2mm}
\begin{equation}
\mathbf{L}~=~\sum_{w}~\mathbf{K}_{w}^* \otimes \mathbf{K}_{w}~=~\sum_{i=1}^r~\gamma_{i} (\mathbf{K}_{i}^* \otimes \mathbf{K}_{i})
\label{eq:L_decomp_noom}
\end{equation}
where $\{\mathbf{K}_{w}\}$ is a set of arbitrary Kraus operators, $\{\sqrt{\gamma_{i}}~\mathbf{K}_{i}\}$ the set of \emph{canonical} Kraus operators defining the CP map, and $r$ the `Kraus-rank' of the CP map. 
It is a well known result that these factors can be computed directly from an SVD of the Choi matrix (the `reshuffled' Liouville matrix); the $i$-th singular value and vector pair correspond to $\gamma_i$ and $\text{vec}(\mathbf{K}_i)$ \citep{Wood2011,Miszczak2011SingularVD}. We illustrate this process in Figure~\ref{fig:cp_map_forms}.
\begin{figure}[ht]
 \centering
    \includegraphics[scale = 0.35]{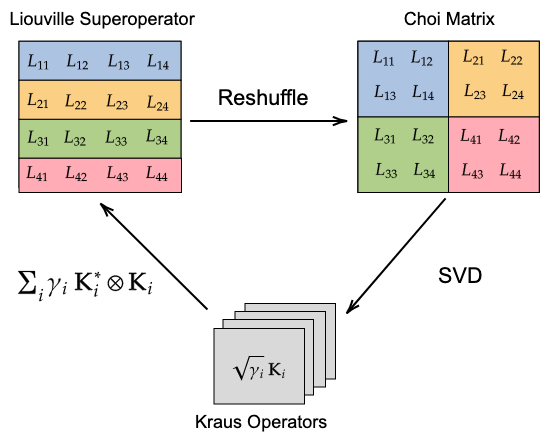}
    \vspace{-2mm}
      \caption{\textbf{Three equivalent formulations of a CP map}: The unique canonical operator sum representation of a CP map can be obtained by performing an SVD of its Choi matrix, which is obtained by reshuffling its Liouville superoperator.\vspace{-4mm}}
      \label{fig:cp_map_forms} 
\end{figure}
The Kraus-rank of a CP map is equal to the rank of the Choi matrix, and is equal to the minimum number of Kraus operators required to express the operation. Since the Liouville superoperator (or the Choi matrix) uniquely defines a CP map, we can use this representations to compare two L-HQMMs.

\paragraph{HQMMs \& NOOMs}
We have shown that $n$-dimensional NOOMs form a subset of $n$-dimensional HQMMs through generalization. Prior work by \citet{srinivasan2018} used `HQMMs' and `NOOMs' somewhat ambiguously, differentiating them primarily by the field over which they are definied ($\mathds{R}$ or $\mathds{C}$). In this paper we have used the original formulation of NOOMs \citep{Zhao2010b} to draw a clearer distinction, whereby NOOMs are simply HQMMs with rank-1 vectorized initial state and Kraus-rank 1 operators.  Particularly, for a fixed latent dimension $n^2$ of the vectorized density matrix, an HQMM allows for a greater diversity of both states and dynamics.

First, note that the valid states of HQMMs are Hermitian PSD matrices with unit trace, also known as mixed density matrices in quantum mechanics \citep{nielsen_chuang_2010}. By contrast, the valid states for NOOMs correspond to the set of pure density operators (with rank $1$). Since these operators encode the probability distribution of the latent state, we see that HQMM states can represent mixture distributions of NOOM states. Formally, the set of rank-1 density matrices are extremal points of the spectraplex defined by arbitrary rank density matrices. This gives us some geometric intuition for why HQMMs have a richer state space than NOOMs.

Second,
HQMMs can have an arbitrary number of Kraus operators per observable while NOOMs are restricted to one to preserve rank-1 states. This indicates that the evolution associated with individual observations in an $n$-dimensional NOOM is restricted to dynamics corresponding to rank $1$ Choi matrices. Thus, an $n$-dimensional HQMMs with arbitrary Kraus rank can encode richer dynamics than an $n$-dimensional NOOM.

\section{Learning HQMMs}
Having characterized the expressiveness of HQMMs, we now turn to the task of learning them from data.

\vspace{-2mm}
\paragraph{The Learning Problem}
We use the negative log-likelihood of the data as our loss function, which can be written as a function of the set of Kraus operators $\{\mathbf{K}_{y,w}\}$ as follows \citep{srinivasan2018}:

\vspace{-4mm}
\begin{small}
\begin{equation}
    \mathcal{L} = -\ln \text{ tr}\left(\sum_{w}\mathbf{K}_{y_n,w}\ldots \left(\sum_{w}\mathbf{K}_{{y_1},w}~\pmb{\rho}_0~\mathbf{K}_{{y_1},w}^\dagger\right)\ldots \mathbf{K}_{y_n,w}^\dagger\right)
    \label{eq:loss}
\end{equation}
\end{small}
Note that the learned Kraus operators must satisfy the TP constraint $\sum_{y,w} \mathbf{K}^\dagger_{y,w} \mathbf{K}_{y,w} = \mathds{I}$. The problem of learning a set of $N$ trace-preserving $n\times n$ Kraus operators can equivalently be framed as one of learning a matrix $\pmb{\kappa} \in \mathds{C}^{nN\times n}$ on the Stiefel manifold i.e., that satisfy $\pmb{\kappa}^\dagger \pmb{\kappa} = \mathds{I}$, where $\pmb{\kappa}$ can be block-partitioned row-wise into the $N$ Kraus operators that parameterize the HQMM. Both the previous and this paper's approach 
begin with an initial guess $\pmb{\kappa}_0$ with a pre-determined partitioning into the Kraus operators we wish to learn, and iteratively make changes to the guess to maximize the log-likelihood (a function of the Kraus operators).
\paragraph{The Previous Approach} Since $\pmb{\kappa}$ is a matrix with orthonormal columns, any initial guess $\pmb{\kappa}_0$ is a unitary transformation away from the true $\pmb{\kappa}^*$ that maximizes the log-likelihood. The existing method \citep{srinivasan2018} iteratively finds a series of Givens rotations that locally increase the log-likelihood. 
However, a Givens rotation only changes two rows of $\pmb{\kappa}$ at a time, making this approach prohibitively slow for learning large $\pmb{\kappa}$ matrices. Furthermore, since these two rows are picked at random, this approach is not guaranteed to step towards the optimum at every iteration. 
\vspace{-2mm}
\paragraph{Retraction-Based Optimization} We propose directly learning $\pmb{\kappa}$ using a gradient-based algorithm. Note that since $\mathcal{L}$ is a function of complex matrices, the direction of steepest descent corresponds to the gradient with respect to the complex conjugate of the Kraus operators \citep{Hjorungnes2007}. Most existing algorithms that constrain gradient updates on the Stiefel manifold are either projection-like (which re-orthogonalize the naive gradient descent updates) or geodesic-like (which directly generate updates on the manifold itself) \citep{Jiang2013AFO}. We picked the geodesic like algorithm proposed by \citet{Wen2013} as it performed best on our datasets (see Appendix~\ref{app:opt_algo_compare} for details).

Given a gradient $\mathbf{G}$ of the loss function $\mathcal{L}$ with respect to parameters $\pmb{\kappa}$, we wish to find the trajectory $\gamma(\tau)$ for some step size $\tau$ that corresponds to stepping along the direction of the gradient while staying on the Stiefel manifold.
The Wen-Yin algorithm achieves this through \textit{retractions} that smoothly map $\mathbf{G}$ or any point on a manifold's tangent bundle onto the manifold itself, while preserving the descent direction at that point \citep{Absil2007}. The constrained update $\gamma(\tau)$ on the Stiefel manifold with respect to an initial feasible solution $\pmb{\kappa}_0$ is
\vspace{-2mm}
\begin{equation}
\gamma(\tau) = \pmb{\kappa}_0 - \tau \mathbf{U} \left( \mathds{I} + \frac{\tau}{2} \mathbf{V}^\dagger \mathbf{U} \right)^{-1} \mathbf{V}^\dagger \mathbf{\kappa}_0,
\label{eq:stiefel_smw_curve}
\end{equation}
where $\mathbf{U} = [\mathbf{G}~|~\pmb{\kappa}_0]$, $\mathbf{V} = [\pmb{\kappa}_0~|-\mathbf{G}]$, and $\mathbf{G}$ is the gradient at $\pmb{\kappa}_0$. This update requires the inversion of a $2n\times 2n$ matrix. $\gamma(\tau)$ is the trajectory obtained by smoothly retracting the gradient onto the manifold, giving the direction of steepest descent to feasibly optimize Equation ~\ref{eq:loss}. Consequently, Equation \ref{eq:stiefel_smw_curve} guarantees that when $\tau = 0$, it has the same direction as $\mathbf{G}$, and $\gamma(\tau)^\dagger \gamma(\tau) = \mathbb{I}$ for any $\tau$.

This method can be combined with a gradient descent scheme (summarized in Algorithm \ref{alg:learn_hqmm}) to learn feasible parameters for HQMMs. In our experiments with $N = |\mathcal{O}|w$, and for a batch with $m$ sequences of length $l$, we compute the loss using Equation~\ref{eq:loss} in $O(mlwn^3)$ time, perform auto-differentiation, and obtain a retraction using Equation~\ref{eq:stiefel_smw_curve} in $O(|\mathcal{O}|wn^3)$ time.

\begin{algorithm}[ht]
  \caption{Learning HQMMs using Constrained Optimization on the Stiefel Manifold}
  \label{alg:learn_hqmm}
  \textbf{Input:} Training data $\mathbf{Y} \in \mathds{N}^{M \times \ell}$, where $M$ is the $\#$ of data points and $\ell$ is the $\#$ of observed variables in the HQMM\\
  \textbf{Hyperparameters:} $\mathbf{\tau}$ (learning rate), $\alpha$: (learning rate decay), $B$ (number of batches), $E$ (number of epochs)\\
  \textbf{Output:} $\{\mathbf{K}_i\}_{i=1}^{|\mathcal{O}|w}$ 
\begin{algorithmic}[1]
\STATE \textbf{Initialize:} Complex orthonormal matrix on Stiefel manifold $\pmb{\kappa} \in \mathds{C}^{|\mathcal{O}|wn \times n}$ and partition into Kraus operators  $\{\mathbf{K}_i\}_{i=1}^{|\mathcal{O}|w}$, with  $\mathbf{K}_i \in \mathds{C}^{n\times n}$
    
  \FOR{$epoch$ = 1: $E$ }
  \STATE Partition training data $\mathbf{Y}$ into $B$ batches $\{\mathbf{Y}_b\}$
    \FOR{$b$ = 1:$B$}
       \STATE Compute gradient $\mathbf{G}_i \gets$ $\frac{\partial \mathcal{L}}{\partial \mathbf{K}_i^*}$ for batch $\mathbf{Y}_b$ and loss function $\mathcal{L}$
       \STATE Compute $\frac{\partial \mathcal{L}}{\partial \mathbf{\kappa}} = \mathbf{G} \leftarrow \begin{bmatrix}\mathbf{G}_1 &\cdots & \mathbf{G}_{|\mathcal{O}|w} \end{bmatrix}^T$
       \STATE Construct $\mathbf{U} \leftarrow [~\mathbf{G}~|~\pmb{\kappa}~]$, $ \mathbf{V} \leftarrow [~\pmb{\kappa}~|~-\mathbf{G}~]$
       \STATE Update $\pmb{\kappa} \leftarrow \pmb{\kappa} - \tau \mathbf{U} \left( \mathds{I} + \frac{\tau}{2} \mathbf{V}^\dagger \mathbf{U} \right)^{-1} \mathbf{V}^\dagger \pmb{\kappa}$
    \ENDFOR
    \STATE Update learning rate $\tau = \alpha \tau$
    \STATE Re-partition $\pmb{\kappa}$ into $\{\mathbf{K}_i\}$
  \ENDFOR
  
  \STATE \textbf{return} $\{\mathbf{K}_i\}$
\end{algorithmic}
\end{algorithm}
\vspace{-2mm}
\section{Experimental Results}
\vspace{-2mm}
To show the superior performance of the retraction-based algorithm for  constrained optimization on the Stiefel manifold ({\bf COSM}) over the previous Givens Search ({\bf GS}) method in learning HQMMs, we evaluate their accuracy and run-time on two datasets. The first is the synthetic dataset used by \citet{srinivasan2018} (code obtained from Github) that was generated by an HMM. The second is a real-world dataset, on which the GS approach is prohibitively slow; demonstrating the scalability of COSM. In Appendix~\ref{app:hqmm_data}, we also present results where COSM outperforms GS on the synthetic data used by \citet{srinivasan2018} that was generated by an HQMM representing a quantum mechanical process. \footnote{A preliminary version of these experimental results appeared in \citet{adhikary2019learning}. Code available at \texttt{https://github.com/sandeshAdhikary/learning-hqmms-stiefel-manifold}}

\vspace{-2mm}
\paragraph{Training} For all our HQMMs, we use the log-likelihood loss function from Equation \ref{eq:loss}. We initialize the latent state $\pmb{\rho}_0$ as a random Hermitian PSD matrix using the QETLAB toolbox \citep{qetlab}, and $\pmb{\kappa}$ as a random orthonormal matrix. Except for very small models, COSM is fairly robust to random initializations (see Appendix~\ref{app:inits}). We compute the gradient of the loss function with respect to the complex conjugate of the Kraus operators using the Autograd package (which can handle complex differentiation), and vertically stack the gradients of the Kraus operators to construct the gradient $\mathbf{G}$ of the matrix $\mathbf{\kappa}$. To smoothen the trajectory we apply momentum  with $\beta = 0.9$ \citep{Rumelhart1998,Qian1999}, and re-normalize the gradient before and after the momentum update, making the magnitude of updates entirely dependent on step-size. We refer to HQMMs using the tuple $(n,s,w)-$HQMM,  where $n$ is the number of hidden states,  $s$ is the number of possible outputs (earlier denoted $|\mathcal{O}|$),  and $w$ is the number of Kraus operators per output, i.e., the Kraus-rank of the HQMM or dimension of the `environment' variable \citep{srinivasan2018}. Consequently, for an $(n,s,w)-$HQMM we have $\pmb{\kappa} \in \mathds{C}^{nsw \times n}$. We also provide the performance of HMMs trained using the Expectation-Maximization ({\bf EM}) algorithm (with 5 random restarts) for reference. Details of our hyperparameter tuning procedure and computing infrastructure are described in Appendix~\ref{app:hyperparams}.

\vspace{-2mm}
\paragraph{Metrics} On the synthetic HMM dataset, we use a scaled log-likelihood \citep{noomreport, srinivasan2018} independent of sequence length called description accuracy: $DA = f\left(1 + \frac{\log_s P(Y|\mathds{D})}{\ell}\right)$, where $f(\cdot)$ squishes the log-likelihood from $(-\infty, 1]$ to $(-1,1)$ (with $f(x) = \tanh({x/8})$ for $x \leq 0$, and $f(x) = x$ for $x > 0$). When $DA=1$, the model predicted the sequence with perfect accuracy, and when  $DA>0$, the model performed better than random. The error bars represent one standard deviation of the DA scores across many test samples. On the real-world dataset, we report the average accuracy for a classification problem.
\subsection{Synthetic HMM Data}
\vspace{-2mm}
For our first experiment, we generated data using the same synthetic HMM as \citet{srinivasan2018}, with 6 hidden states and 6 possible outputs. We show two things with the experiments on this dataset: 1) COSM finds better optima than GS, and 2) COSM is much faster than GS -- so much so that we could train larger HQMMs than were previously possible. We also investigate the effects of increasing model size by adding latent states ($n$) versus increasing the Kraus-rank ($w$). 

We used the same 20 training and 10 validation sequences of length 3000 used by \citet{srinivasan2018}, splitting up each sequence into 300 sequences and use a burn-in of 100. We trained HQMMs using the COSM approach for 60 epochs, and evaluated the model with the highest validation DA score on the test set. The results for this model are shown in  Figure~\ref{fig:synthetic_hmm_w1_results} and Figure~\ref{fig:synthetic_hmm_n5_results}.

\begin{figure}[ht]
    \centering
    \begin{subfigure}[t]{0.5\textwidth}
        \centering
        \includegraphics[scale = 0.28]{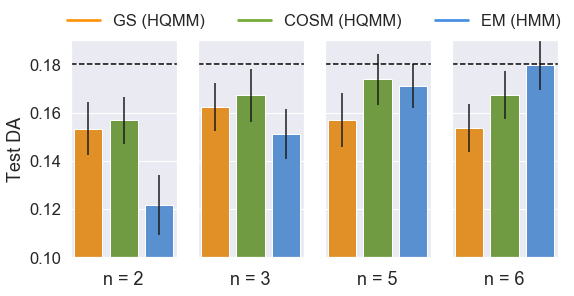}
                \caption{$w=1$}
      \label{fig:synthetic_hmm_w1_results}
    \end{subfigure}
    \begin{subfigure}[t]{0.5\textwidth}
        \centering
        \includegraphics[scale = 0.28]{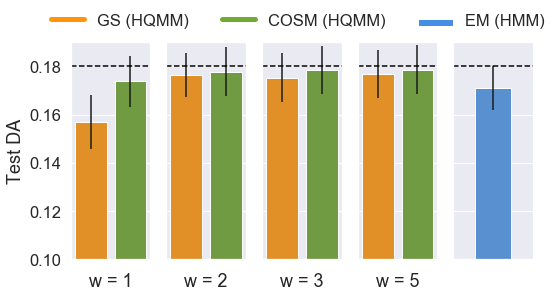}
        \caption{$n=5$}
      \label{fig:synthetic_hmm_n5_results}
    \end{subfigure}
    \caption{{\bf Test Set Performances on the Synthetic HMM Data}: The dashed line represents the test set performance of the true model (a (6,6)-HMM) that generated the data.}
\end{figure}

\vspace{-2mm}
\paragraph{COSM finds better optima than GS} As shown in Figure~\ref{fig:synthetic_hmm_w1_results}, HQMMs (with $w=1$) learned using COSM achieve better optima than HQMMs learned using GS for all $n$. As described in Section~\ref{sec:hqmms}, these models are essentially complex-valued NOOMs. We also confirm that as noted in \citet{srinivasan2018}, small HQMMs ($n \leq 5$) can model this data better than small HMMs, although this doesn't hold for $n=6$. However, we can take advantage of the additional Kraus-rank hyperparameter $w$ available to HQMMs to further improve performance, as shown in Figure~\ref{fig:synthetic_hmm_n5_results} for $(5,6,w)-$HQMMs (varying $w$). Also note that the number of parameters for an HQMM scales faster than for an HMM.
\vspace{-2mm}
\paragraph{COSM is much faster than GS} In Figure \ref{fig:hmm_results_training_times}, we plot the test set DA versus CPU training time for the smallest and largest models trained. To ensure a fair comparison, we train both approaches on sequences of length $300$ and a batch size of 30. Note that we pre-tune hyperparameters on the validation set, and the graphs show the changing test DA as the models are trained with these hyperparameters (test DAs were not used to tune hyperparameters).

\begin{figure*}[ht]
    \centering
    \includegraphics[scale=0.3]{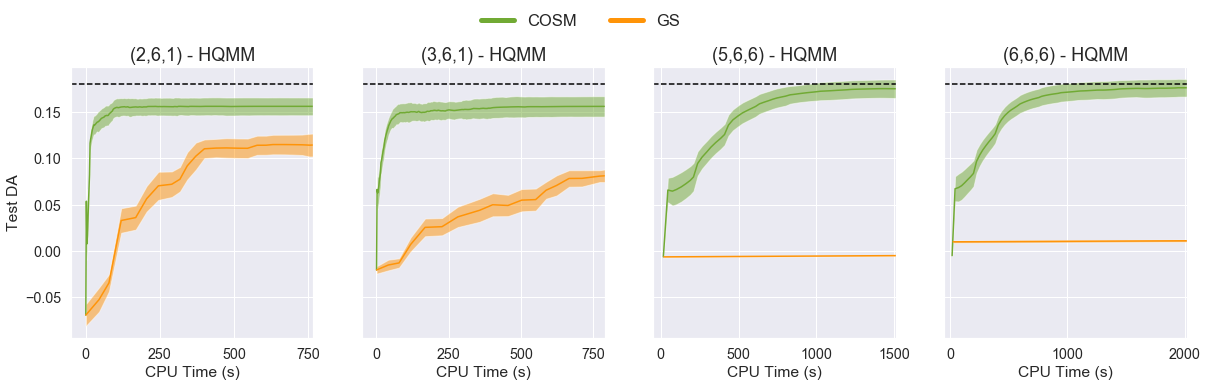}\vspace{-2mm}
    \caption{\textbf{COSM Learns More Accurate Models Faster than GS}: Test DA versus training time for various $(n,s,w)$-HQMMs trained on the synthetic HMM data. COSM converges to a better optimum faster than GS for all models; the dashed line represents the DA of the true data generating model.\vspace{-4mm}}
    \label{fig:hmm_results_training_times}
\end{figure*}

For all models, we see that COSM converges much faster than GS, and the difference in both speed and accuracy is especially pronounced for the larger models; COSM converges within a few hundred seconds, while GS yields very poor solutions even after 2000 seconds.
As the GS method can take days to converge for large models, we could not directly calculate a precise speedup. 

\citet{srinivasan2018} proved that a $(6,6,6)$-HQMM should be sufficient to fully model a $(6,6)$-HMM, but the GS method was too slow to train this model. With COSM, we are able to show that this theoretical guarantee holds in practice. In fact, we find that in practice a $(5,6,3)-$HQMM is sufficient to model our $(6,6)-$HMM.

\subsection{Splice Dataset}\vspace{-2mm}
For our second experiment, we use the real-world splice dataset \citep{Dua2017,splice_dataset}
consisting of DNA sequences of length 60, each element of which represents one of four nucleobases: Adenine (A), Cytosine (C), Guanine (G), and Thyamine (T). A DNA sequence typically consists of information encoded in sub-sequences (exons), that are separated by superfluous sub-sequences (introns). The task associated with this dataset is to classify sequences as having an exon-intron (EI) splice, an intron-exon (IE) splice, or neither (N), with 762, 765, and 1648 labeled examples for each label respectively. In addition to A, C, T and G, the raw dataset also contains some ambiguous characters, which we filter out prior to training. Our goal in this experiment is to demonstrate that we can use COSM to train HQMMs on real-world datasets which would have been too slow to train using GS.

We train a separate model for each of the three labels, and during test-time, choose the label corresponding to the model that assigned the highest likelihood to the given sequence. We train HQMMs using the COSM method and HMMs with the EM algorithm (with 5 random restarts) for reference. In Figure~\ref{fig:splice_results}, we report the average classification accuracies across all labels obtained with 5-fold cross validation. For reference, a random classifier achieves around $33.3\%$ accuracy.

\vspace{-2mm}
\begin{figure}[ht]
    \centering
    {\transparent{0.9}\includegraphics[scale = 0.3]{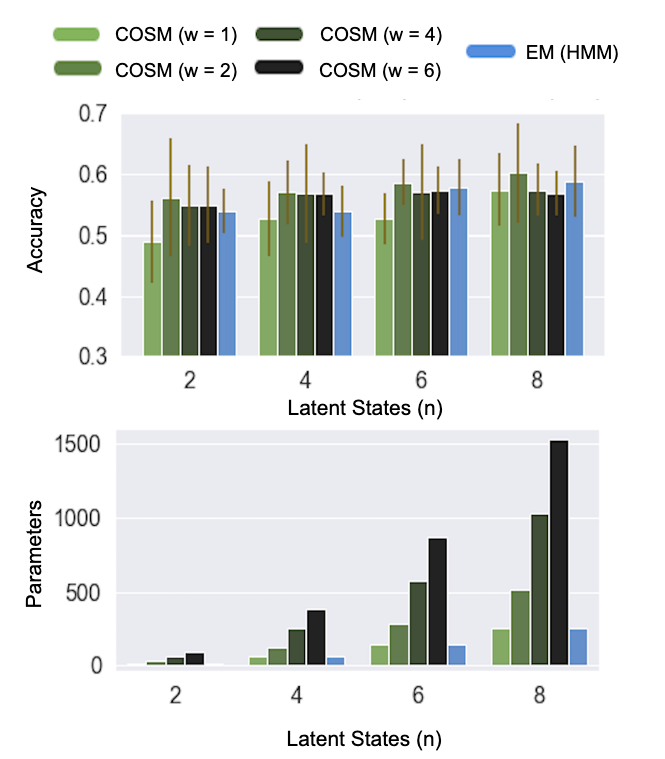}}\vspace{-2mm}
    \caption{\textbf{Average 5-fold Test Set Performance on the Splice Dataset} Test set accuracies (left) and number of parameters (right) for various HQMMs and HMMs trained using the COSM and EM algorithms respectively. Errorbars in the left graph represent the mean standard deviation across labels over the 5 folds. }
    \label{fig:splice_results}
\end{figure}
Note that 5-fold cross-validation is prohibitively time consuming for GS, even for models with a modest number of parameters. However, we are able to learn these  HQMMs with COSM. We also see that, (as before) there is a sizable marginal gain in DA when going from $w=1$ to $w=2$, with the benefits of increasing $w$ further being less clear. However unlike the previous experiment, we still see persistent gains by increasing $n$. Interpreting this in conjunction with the results in the previous section suggests that we have to tune both $n$ and $w$ depending on the dataset. We also find that for a given number of hidden states, COSM is able to learn an HQMM that outperforms the corresponding HMM, although this comes at the cost of a rapid scaling in the number of parameters. 
\section{Conclusion}\vspace{-2mm}
We showed that HQMMs are OOMs that generalize NOOMs, and that unlike prior approaches that avoid the NPP by design, HQMMs are able to model arbitrary HMMs as well. HQMMs expand the convex cone of valid states from rank-1 PSD matrices in (complex valued) NOOMs to arbitrary rank Hermitian PSD matrices. We also formulated the unique Liouville representation of an HQMMs, which allows direct comparison between models,
and also simplifies theoretical analysis connecting them to general OOMs. Future work could focus on identifying the exact relationship bewtween NOOMs and HMMs, and whether arbitrary OOMs can be converted to HQMMs.

We also introduced a retraction-based learning algorithm that directly constrains gradient updates to the Stiefel manifold to learn feasible HQMMs, and presented experimental results on a synthetic and a real-world dataset. In the process, we showed that the proposed algorithm outperforms the prior approach in terms of both speed and accuracy, and so were able to train HQMMs that were previously too large to train. This also  suggests that directly optimizing the parameters is a better strategy  than finding small, local unitary  rotations of the matrix on the Stiefel manifold. One downside is the rapid scaling of parameters in HQMMs, and it would be interesting to investigate approximations that may produce similar performance with far fewer parameters. It would also be useful to dynamically learn the Kraus-rank $w$ instead of tuning it as a hyperparameter. Other future work could develop new QGM models defined via Kraus operators, which can be learned using our approach.

\bibliographystyle{apalike}
\bibliography{references.bib}

\clearpage
\newpage 
\begin{appendix}
\section{Experiment on Synthetic HQMM Data}
\label{app:hqmm_data}
As an additional experiment on a purely quantum mechanical dataset, we compared the COSM and GS methods on data generated using the synthetic HQMM with 2 hidden states and 6 possible outputs in \citet{srinivasan2018}. The  data generation process is inspired by the well known Stern-Gerlach experiment \citep{gerlach1922experimentelle} in  quantum mechanics, and at least 4 hidden states are required to model it. \citet{srinivasan2018} demonstrated that HQMMs \emph{learned} from such synthetic data showed in practice the same benefits that held in theory. Our goal is to verify that the COSM method performs at least as well as the GS method on a dataset well-suited to the HQMM model class.

We used the same synthetic dataset used by \citet{srinivasan2018}, with 20 training and 10 validation sequences of length 3000. We further split up each sequence into 300 sequences and use a burn-in of 100, instead of training on 3000-length sequences with a burn-in of 1000. This reduced training time without impacting accuracy or the amount of training data processed. We trained HQMMs using the COSM approach for 60 epochs, and saved the model that yielded the highest DA score on the validation set; we used this model to evaluate on the test set of 10 sequences of length 3000 (with burn-in 1000). The results for this model are shown in  Figure~\ref{fig:synthetic_hqmm_results}.
We see that the COSM method achieves slightly better DA compared to the GS method. We confirm that as seen in \citet{srinivasan2018}, we need a $6-$state HMM to model this $2-$state HQMM.
\begin{figure}[ht]
 \centering
    \includegraphics[scale = 0.3]{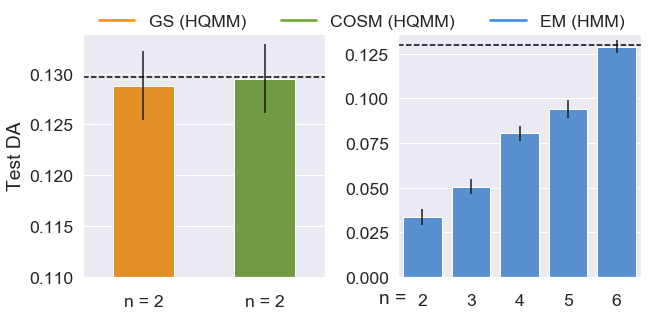}
      \caption{\textbf{Test Set Performance on the Synthetic HQMM Data}: The dashed line represents the test set performance of the true model that generated the data. The GS and COSM methods were used to learn ($2,6,1$)-HQMMs, while EM was used to learn HMM models with varying number of hidden states ($n$). A $6-$state HMM model was needed to match a $2-$state HQMM.}
      \label{fig:synthetic_hqmm_results} 
\end{figure}

\section{Updates on the Stiefel Manifold}
\label{app:opt_algo_compare}
Algorithms that constrain parameters on the Stiefel manifold generally are either projection-like (which re-orthogonalize the naive gradient descent updates) or geodesic-like (which directly generate updates on the manifold itself). Among geodesic-like algorithms, those proposed by \citet{Wen2013} and \citet{Jiang2013AFO} are the current state-of-the-art approaches. In the regime of tall-and-skinny matrices in our problem, these two are theoretically equivalent and have the same computational complexity $O(7Nn^2)$, where $n$ is the latent dimension and $N = sw$. By comparison, the canonical gradient projection algorithm has a slightly lower computational complexity of $O(3Nn^2)$ \citep{Jiang2013AFO}. We compared these three update schemes to project or retract gradients onto the Stiefel manifold. The exact update schemes for all three methods can be found in \citet{Jiang2013AFO}.

We trained $9$ HQMM models for both the synthetic HQMM and synthetic HMM datasets using these 3 update schemes. As shown in the results in Figures~\ref{fig:alt_updates_hqmm} and \ref{fig:alt_updates_hmm}, the three methods are very similar both in terms of speed and the final solution quality for our benchmark datasets. Since the Wen-Yin update was slightly faster, especially for larger models on the synthetic HQMM data, we used it over the alternatives.
\begin{figure*}[ht]
    \centering
    \begin{subfigure}[t]{0.5\textwidth}
        \centering
            \includegraphics[width=0.8\textwidth]{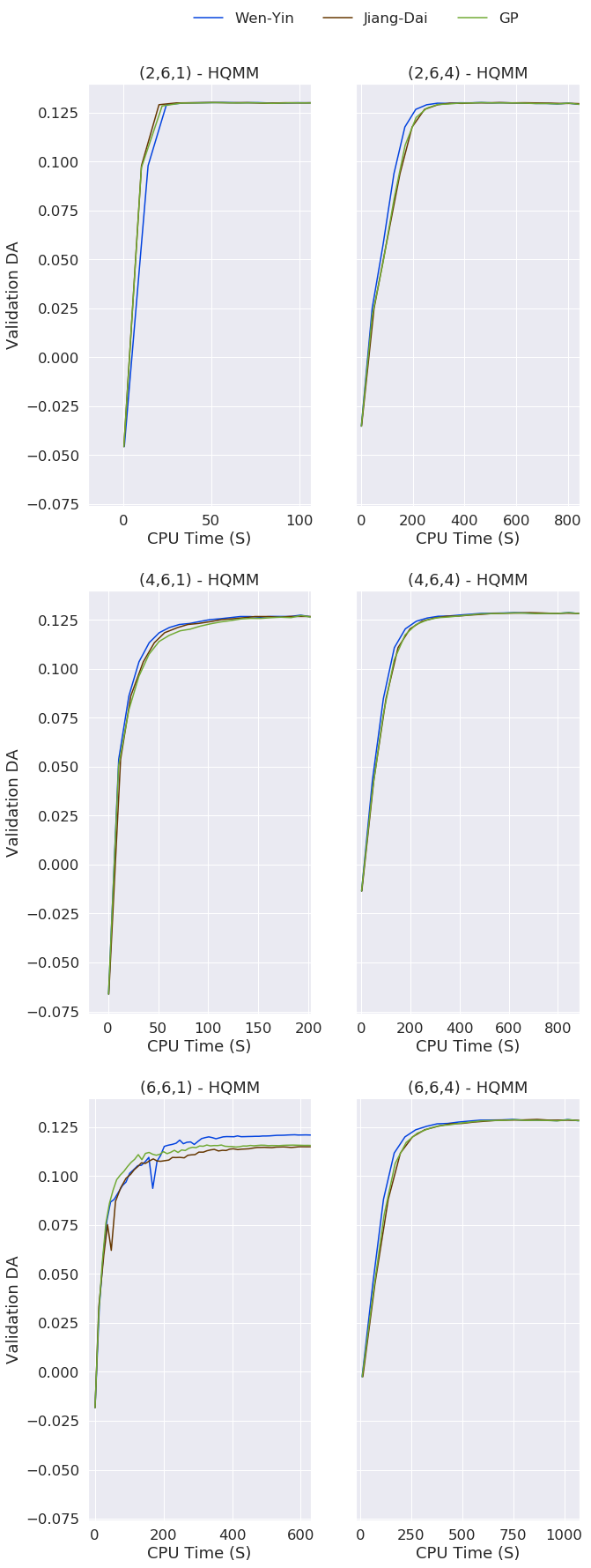}
              \caption{Results for the Synthetic HQMM Data}
          \label{fig:alt_updates_hqmm}
    \end{subfigure}%
    ~ 
    \begin{subfigure}[t]{0.5\textwidth}
        \centering
    \includegraphics[width=0.8\textwidth]{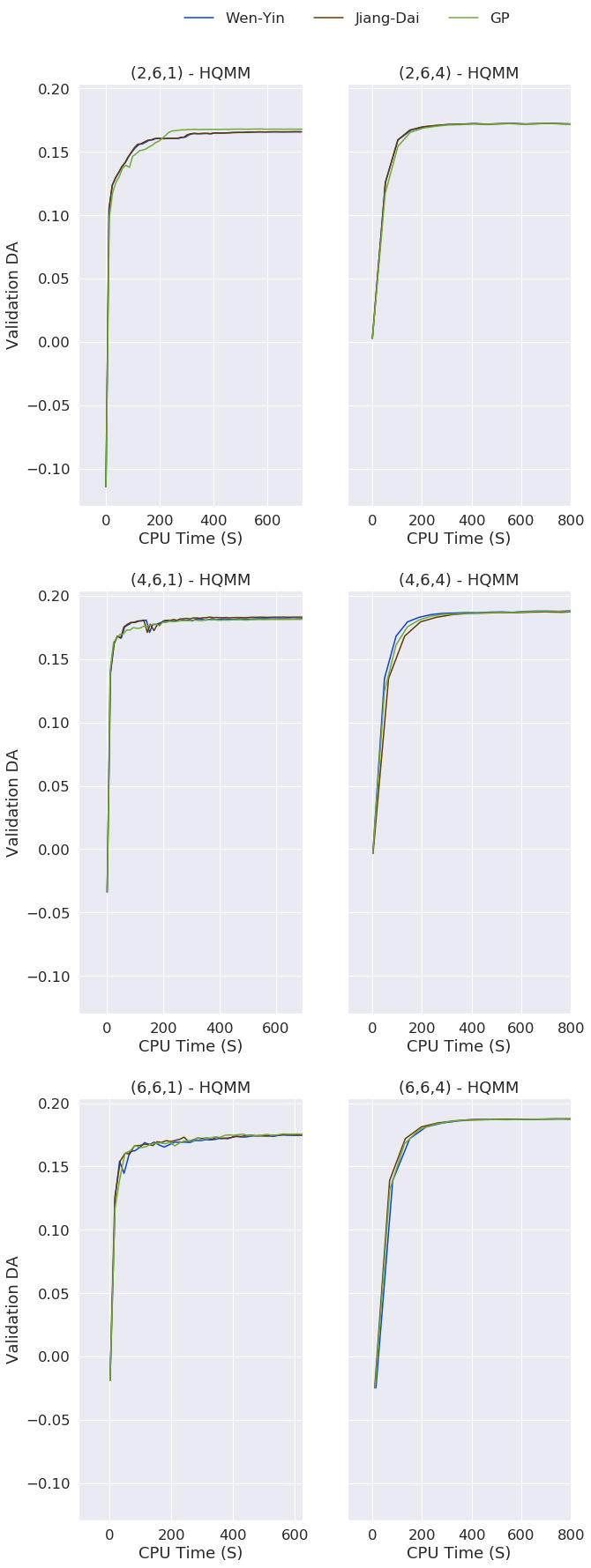}
      \caption{Results for the Synthetic HMM Data}
  \label{fig:alt_updates_hmm}
    \end{subfigure}
    \caption{\textbf{Alternative Schemes to Constrain Updates on the Stiefel Manifold} Validation set accuracies obtained for HQMMs trained using different update schemes. All schemes provide similar speed and accuracy, but the Wen-Yin update outperforms the others by a small margin.}
    \vspace{-4mm}
\end{figure*}

\section{Sensitivity to Initialization}
\label{app:inits}
The COSM algorithm begins with an initial guess of the optimal parameters $\pmb{\kappa}$ and a random intial density matrix $\pmb{\rho}$. By `burning-in' a reasonable number of initial entries in sequences, we minimize the effect of randomly initializing $\pmb{\rho}$. To investigate the sensitivity of COSM to initializations of $\pmb{\kappa}$, we trained models on the synthetic HQMM and HMM datasets over $3$ random seeds. As shown in the results in Figure~\ref{fig:random_inits_hqmm} and \ref{fig:random_inits_hmm}, COSM is sensitive to random initializations for the smallest $(2,6,1)$ model, but the variance in DA scores quickly decrease with an increase in model size, both as a function of $n$ and $w$. We observe even lower variance across different initializations for the synthetic HMM data in Figure \ref{fig:random_inits_hmm}.

\begin{figure*}[ht]
    \centering
    \begin{subfigure}[t]{0.5\textwidth}
        \centering
            \includegraphics[width=0.85\textwidth]{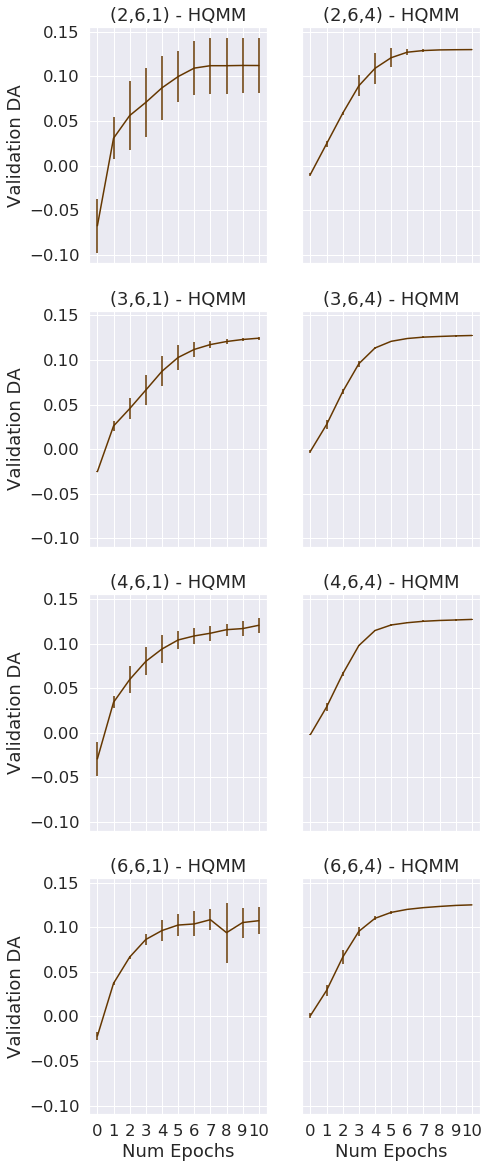}
              \caption{Synthetic HQMM Data}
          \label{fig:random_inits_hqmm}
    \end{subfigure}%
    ~ 
    \begin{subfigure}[t]{0.5\textwidth}
        \centering
    \includegraphics[width=0.85\textwidth]{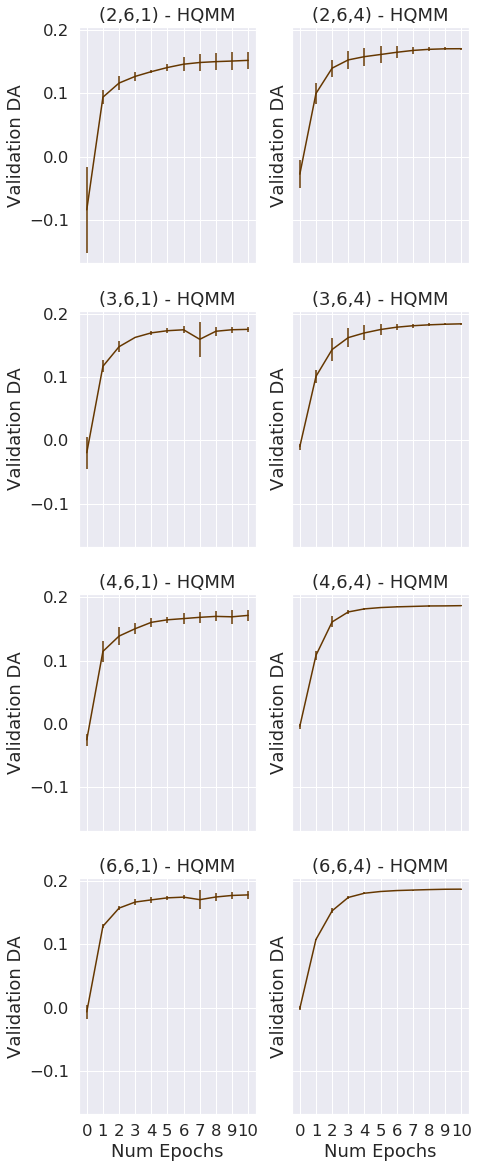}
      \caption{Synthetic HMM Data}
  \label{fig:random_inits_hmm}
    \end{subfigure}
    \caption{\textbf{COSM's Sensitivity to Random Initializations of $\kappa$} Validation set accuracies obtained across $10$ epochs for HQMMs trained on $3$ different random initializations. COSM is sensitive to $\pmb{\kappa}$ initialization for the smallest models, but is fairly robust for larger models.}
\end{figure*}

\section{Hyperparameter Selection}
\label{app:hyperparams}
To facilitate a clear comparison with GS, we used the same batch size as in \citet{srinivasan2018}, and tuned the step-size $\tau$ and decay rate $\alpha$ for all HQMM models. We started by manually tuning models, and identified that all models tended to converge to good solutions with the following hyperparameters: $\tau = 0.75$ and $\alpha = 0.92$ for the synthetic datasets, and 
$\tau = 0.8$ and $\alpha  = 0.9$ for the splice dataset. We trained baseline models using these parameters, and then randomly searched for better configurations around these values. 

For the synthetic datasets, we fixed the batch size at $20$ and randomly sampled $\tau$ between $0.55$ and $0.95$, and $\alpha$ between $0.9$ and $0.99$. As we wanted to explore many hyperparameter settings, we only trained on $3$ random batches in every epoch. For the splice dataset, we fixed the batch size at $200$ and randomly sampled $\tau$ between $0.7$ and $0.9$ and $\alpha$ between $0.88$ and $0.92$. Since each splice model required learning three separate HQMMs across multiple folds, we tested fewer hyperparameter settings across a smaller search space. We also trained on a single random batch every epoch across $2$ folds.

Given the large number of models that we needed to evaluate, we used the Hyperband scheduling technique \citep{Li2017HyperbandAN} to quickly sample through many hyperparameter configurations. For each model, we began by running $3$ epochs for each of the $k$ randomly selected configurations, and removed $k/3$ of them with the lowest validation DA scores. In the next round, we ran the remaining configurations for a larger number of iterations, and again removed the bottom third of the configurations with the lowest scores. We repeated this strategy until only one configuration remained, and saved the one with the highest validation DA throughout the tuning protocol. We searched across $27$ and $9$ random configurations for the synthetic and the splice datasets respectively. As an example, for the synthetic datasets we trained $27$ models for $3$ epochs, followed by the $9$ best models for $9$ epochs, followed by the $3$ best models for $9$ epochs, and the final best model for $27$ epochs. In Table~\ref{table:hyperparams}, we report the hyperparameters obtained through Hyperband that outperformed the default configuration. For models not listed in the table, the default configuration resulted in the best performance.

All our experiments were performed on a desktop with 8 Intel Core i7-7700K 4.20 GHz CPUs, and 31.3 GB RAM. All models are trained in MATLAB, but the gradient computation happens in Python.

\begin{table}[]
\centering
\caption{\textbf{Hyperparameter Selection} The best performing step sizes ($\tau$) and decay rates ($\alpha$) for various COSM models. For models not listed here, the default hyperparameters $(\tau = 0.75, \alpha = 0.92)$ and $(\tau = 0.8, \alpha = 0.9)$ yielded the best results for the synthetic datasets and the splice dataset respectively.}

\begin{tabular}{|c|c|c|c|c|c|}
\hline
\multicolumn{1}{|c|}{\bf Dataset}                        & $n$ & $s$ & $w$ & {\bf $\tau$} & {\bf $\alpha$} \\ \hline
\multicolumn{1}{|l|}{Synthetic HQMM}                 & 2 & 6 & 1 & 0.75      & 0.92       \\ \hline
\multicolumn{1}{|l|}{\multirow{5}{*}{Synthetic HMM}} & 2 & 6 & 1 & 0.95      & 0.99       \\  
\multicolumn{1}{|l|}{}                               & 4 & 6 & 6 & 0.95      & 0.96       \\ 
\multicolumn{1}{|l|}{}                               & 5 & 6 & 1 & 0.55     & 0.96      \\ 
\multicolumn{1}{|l|}{}                               & 5 & 6 & 2 & 0.95      & 0.98       \\ 
\multicolumn{1}{|l|}{}                               & 5 & 6 & 6 & 0.95      & 0.99       \\ \hline
{\multirow{7}{*}{Splice}}                             & 2 & 4 & 1 & 0.70          &   0.90         \\  
                                                     & 2 & 4 & 2 &  0.85         &     0.92       \\  
                                                     & 2 & 4 & 6 &     0.85      & 0.92            \\ 
                                                     & 4 & 4 & 1 &    0.90       & 0.92            \\  
                                                     & 4 & 4 & 4 &    0.90       & 0.90            \\  
                                                     & 6 & 4 & 4 &    0.70     &      0.90        \\ 
                                                     & 8 & 4 & 1 &0.90           &       0.90      \\ 
                                                     \hline
\end{tabular}
\label{table:hyperparams}
\end{table}

\section{Estimating Speedup}
\label{app:speedup}
Since the GS method can take days to converge to the final solution for large models such as $(6,6,6)$-HQMM, it was not feasible to compute a direct speed up comparing its convergence time to COSM across most models. Thus, we estimate the speed-up offered by COSM by fitting a linear model to the DA trajectory of models learned by the GS method. Specifically, for a given HQMM model, we train both COSM and GS on the synthetic HMM data until one of them converges within a tolerance of $10^{-5}$ in DA scores. Since COSM always converges first, we take the DA scores achieved by GS in its last $10$ steps and fit a linear model to it. We then extrapolate this linear model to estimate the time it would take for GS to reach some fraction of the solution DA reached by COSM. Note that a linear fit is an optimistic assumption of GS convergence time, meaning we are going to \emph{understate} how much faster COSM is compared to GS. Finally, we estimate the speed up offered by COSM as the ratio of the (estimated) convergence time for GS and the actual convergence time for COSM. In Figure~\ref{fig:speed_up}, we plot this estimated speed up with varying number of parameters (both as functions of $n$ and $w$) for different solution fractions. For a solution fraction of $1$, we record speedups greater than $150\times$ for the largest HQMMs trained. Furthermore, COSM offers comparable increase in speed up as parameters grow either by virtue of increasing the number of latent states $n$ or the Kraus-rank $w$.

\begin{figure*}[ht]
  \centering
    \includegraphics[width=0.65\textwidth]{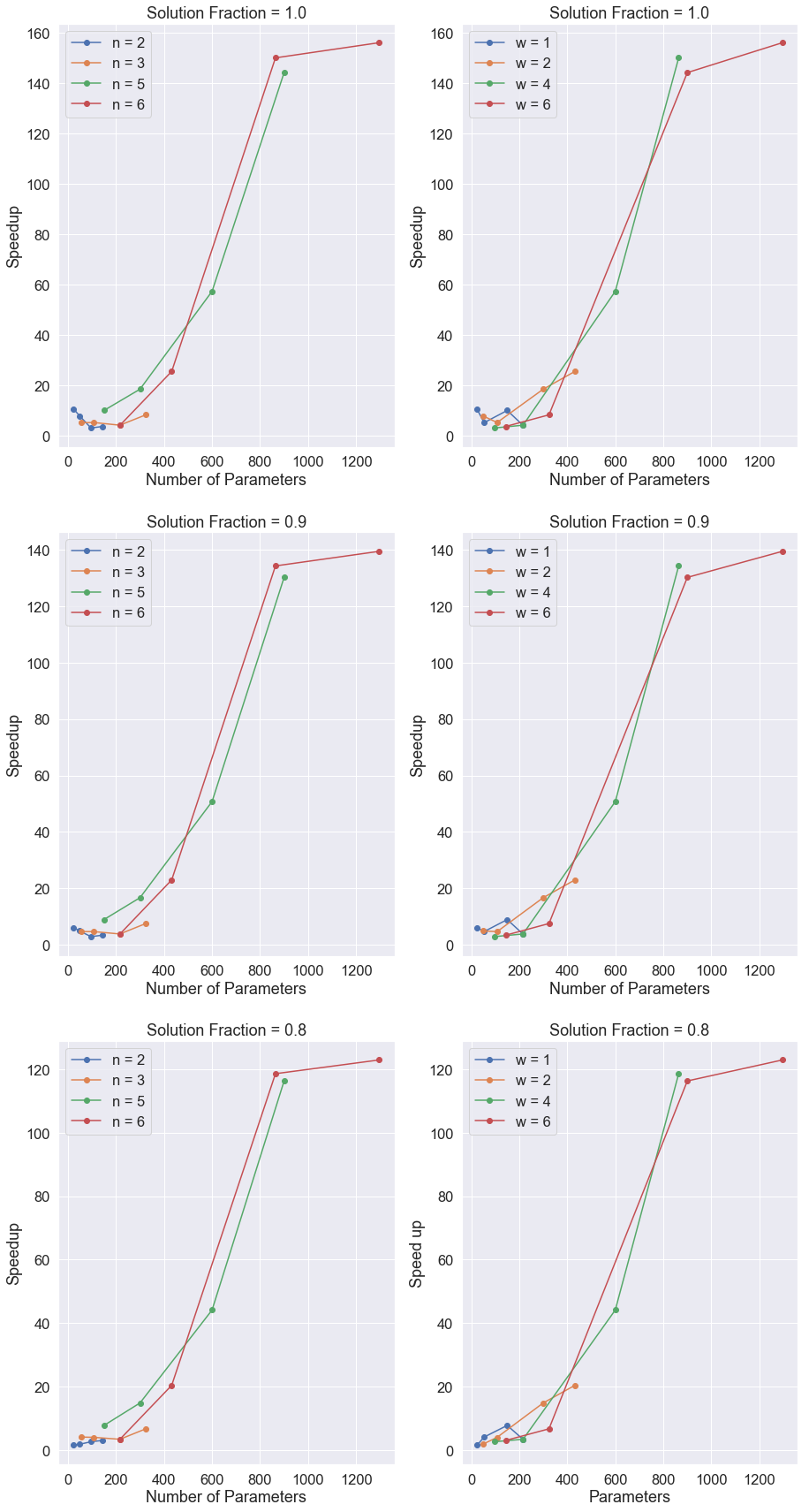}
      \caption{\textbf{Estimated Speedup of COSM over GS}: Estimated speedups of COSM over GS for various solution fractions. As seen in the plots for solution fraction of $1$, GS can take more than $150$ times the convergence time for COSM to reach the latter's final solution quality.}
  \label{fig:speed_up}
\end{figure*}

\end{appendix}

\end{document}